\documentclass[lettersize, journal]{IEEEtran}

\IEEEoverridecommandlockouts

\usepackage{graphics} 
\usepackage{epsfig} 
\usepackage{times} 
\usepackage{amsmath} 
\usepackage{stackengine}
\usepackage{amssymb} 
\usepackage{amsthm}
\usepackage{multirow}
\usepackage{upgreek}
\usepackage{lettrine}
\usepackage{subfigure}
\usepackage{mathtools,amssymb,lipsum}
\usepackage{cuted}
\usepackage{multirow}  
\usepackage{color}
\usepackage{cite}
\usepackage{lipsum}
\usepackage{cuted}
\usepackage{caption}
\usepackage{enumitem}
\usepackage{amsmath,amsfonts}
\usepackage{mathtools}
\usepackage{amssymb}
\usepackage{array}
\usepackage{textcomp}
\usepackage{stfloats}
\usepackage{url}
\usepackage{setspace}
\usepackage{graphicx}

\newtheorem{theorem}{Theorem}

\newtheorem{lemma}{Lemma}
\newtheorem{assumption}{Assumption}

\newtheorem{remark}{Remark}

\newtheorem{problem}{Problem}

\definecolor{LB}{RGB}{65,105,225}
\definecolor{LO}{RGB}{210,127,0}
\definecolor{LP}{cmyk}{0, 0.7808, 0.4429, 0.1412}
\definecolor{LG}{rgb}{0.13, 0.55, 0.13}
\definecolor{LR}{rgb}{0.77, 0.01, 0.2}

\begin{document}


\title{Safe Circumnavigation of a Hostile Target Using Range-Based Measurements}

\author{Gaurav Singh Bhati, Arukonda Vaishnavi, and Anoop Jain
\thanks{The authors are with the Department of Electrical Engineering, Indian Institute of Technology Jodhpur, 342030, India (e-mail: {m22ee003@iitj.ac.in}; {vaishnavi.2@iitj.ac.in}; {anoopj@iitj.ac.in}).}}

\maketitle
	

\begin{abstract}
Robotic systems are frequently deployed in missions that are dull, dirty, and dangerous, where ensuring their safety is of paramount importance when designing stabilizing controllers to achieve their desired goals. This paper addresses the problem of safe circumnavigation around a hostile target by a nonholonomic robot, with the objective of maintaining a desired safe distance from the target. Our solution approach involves incorporating an auxiliary circle into the problem formulation, which assists in navigating the robot around the target using available range-based measurements. By leveraging the concept of a barrier Lyapunov function, we propose a novel control law that ensures stable circumnavigation around the target while preventing the robot from entering the safety circle. This controller is designed based on a parameter that depends on the radii of three circles, namely the \emph{stabilizing circle}, the \emph{auxiliary circle}, and the \emph{safety circle}. By identifying an appropriate range for this design parameter, we rigorously prove the stability of the desired equilibrium of the closed-loop system. Additionally, we provide an analysis of the robot's motion within the auxiliary circle, which is influenced by a gain parameter in the proposed controller. Simulation and experimental results are presented to illustrate the key theoretical developments.
\end{abstract}



\section{Introduction}\label{section_1}
With the rapid advancements in defense technologies, modern targets have become increasingly sophisticated and may pose severe threats to their adversaries. While autonomous vehicles such as unmanned ground vehicles (UGVs) and unmanned aerial vehicles (UAVs) have demonstrated remarkable flexibility in navigating challenging and hazardous routes \cite{qin2019autonomous, petravcek2021large}, a significant challenge that continues to test is their ability to safely traverse areas with adversarial or hazardous targets without compromising their survival \cite{deghat2014localization, cao2021safe}. For instance, when tracking an explosive target, a UAV must maintain a safe distance to avoid triggering an explosion \cite{zengin2007real}. Similarly, when navigating around dangerous chemical spills or radioactive zones, a UGV must keep a safe distance to prevent contamination \cite{brinon2015distributed,demetriou2020navigating}. Moreover, signals from global positioning system (GPS) are often unreliable or unavailable in such harsh operating areas, requiring robotic systems to rely on local measurements for navigation \cite{cao2015uav, qin2019autonomous}. This paper contributes to addressing these challenges by proposing a controller that uses only range-based measurements$-$specifically, range and range-rate information$-$and ensures the robot's safety from the hostile target at all times.

Target circumnavigation \cite{cao2015uav,matveev2017tight,wang2024target} or standoff tracking \cite{oh2015coordinated} refers to the problem of steering an autonomous robot around a target at a prescribed distance. A large body of literature exists in this direction, focusing on different sensing capabilities of the robot, such as position or distance-based measurements \cite{zheng2015distributed,shi2021distributed,shames2011circumnavigation}, bearing-only measurements \cite{deghat2014localization,cao2021safe,zheng2015distributed,zou2020distributed}, and range-based measurements \cite{cao2015uav,wang2024target,cao2013circumnavigation,wang2021mobile}. In practice, implementing position-based approaches often relies on GPS signals that may be limited in harsh environments. On the other hand, bearing-based and range-based approaches have gained special attention due to their simplicity in practical implementation. Bearing-based methods can be implemented using monocular camera systems \cite{zheng2015enclosing}, while range-based methods can use stereo vision sensors or LiDAR to measure the range to the hostile target \cite{hrabar2012evaluation}, with range-rate information approximated from the range data \cite{wang2021mobile}. It is important to note that circumnavigation is essentially equivalent to the circular path tracking problem, given the limited sensing capabilities of the robot \cite{brinon2019circular,jain2019trajectory}.

However, the majority of existing works, including those previously mentioned, address target circumnavigation problems without imposing constraints on the robot's motion near the target. While a bearing-based safe circumnavigation approach is discussed in \cite{li2019safe}, it is limited to a single integrator agent model. To the best of the authors' knowledge, the problem of safe circumnavigation using range-based measurements for (nonholonomic) robotic systems remains unexplored in the literature. To address this problem, we leverage the concept of barrier Lyapunov function (BLF), which extends the idea of traditional control Lyapunov functions to systems with constraints. A BLF approaches infinity as its argument nears the boundary of the constraining set, thereby shaping the stabilizing controller to prevent the system from violating the desired operating constraints. Various BLF formulations have been proposed in different contexts: logarithmic BLFs for handling symmetrical and asymmetrical output constraints \cite{tee2009barrier}, universal barrier functions to limit the value of the Lyapunov function at boundaries \cite{jin2018adaptive}, re-centered BLFs for collision avoidance and connectivity preservation in multi-agent systems \cite{panagou2015distributed}, and parameterized BLFs for managing uncertainties in communication and measurements \cite{han2019robust}. An avoidance shell technique is employed in \cite{ballaben2024lyapunov} to address simultaneous avoidance and stabilization for control-affine nonlinear systems. In this paper, we adopt the logarithmic BLF introduced in \cite{tee2009barrier} due their simplicity in handling constraints.

\emph{Main Features and Contributions:} The proposed approach is based on the consideration of an auxiliary circle $\mathcal{C}_a$ in the problem formulation, which guides the robot to stabilize around the desired circular trajectory using only range and range-rate measurements. This concept is inspired by the seminal work \cite{cao2015uav}, which does not address the issue of safety from a hostile target. Our control design consists of two main steps: (i) Firstly, we introduce a bounded scalar function, $\eta$, which attains minimum value at the desired equilibrium, (ii) Based on $\eta$, we then propose a stabilizing control law using the BLF that ensures $\eta$ remains bounded by a design parameter, $\delta$, at all times. Under standard assumptions on the radii of the three circles, we propose an appropriate choice of $\delta$, with which, the proposed controller guarantees that the robot avoids the safety circle around the hostile target and circumnavigates it on the desired circular path. We further characterize the robot's motion inside the auxiliary circle concerning the gain parameter $\kappa$ appearing in the control law. We obtain sufficient conditions on $\kappa$ such that the robot enters the $\mathcal{C}_a$ circle at most once and multiple times. In either case, the robot always avoids the safety circle. Finally, the proposed control strategy was validated experimentally using the Khepera IV differential drive ground robot in a motion capture (MoCap) system comprising overhead cameras. Unlike \cite{wang2024target}, which uses ultra-wideband (UWB) wireless sensors mounted on both the robot and the target for range measurements, our approach does not assume the placement of a sensor on the (hostile) target, which may not be feasible in practice.  

\emph{Preliminaries:} We denote by $\mathbb{R}$ the set of real numbers and by $\mathbb{R}_+$ the set of non-negative real numbers. A differentiable map $\mathfrak{f}: \mathcal{D} \to\mathbb{R}$, $\mathcal{D} \subseteq \mathbb{R}^N$, has a gradient $\nabla \mathfrak{f} = \left[{\partial \mathfrak{f}}/{\partial x_1}, \cdots, {\partial \mathfrak{f}}/{\partial x_n}\right]^\top$ for $x \in \mathbb{R}^n$. For any $\nu > 0$, the following inequalities hold:
\begin{align}
	\label{shafer_inequality} \hspace*{-0.2cm}  &\text{Shafer's inequality \cite{shafer1966elementary}}: \frac{3}{1 + 2\sqrt{1 + \nu^2}} < \frac{\tan^{-1}(\nu)}{\nu} < 1,\\
	\label{logarithmic_inequality} \hspace*{-0.2cm}  &\text{Logarithmic inequality \cite{love198064}}: 1 - \frac{1}{\nu} \leq \ln (\nu) \leq \nu - 1.
\end{align}
Next, we state the following convergence lemma involving BLF that will be helpful in the subsequent analysis.
\begin{lemma}[\hspace{-.1pt}Convergence under BLF {\cite[pp. 919$-$920]{tee2009barrier}}]\label{lem_blf}
	For any positive constant $\rho$, let $\mathcal{Z} \triangleq \{z \in \mathbb{R} \mid -\rho < z < \rho\} \subset \mathbb{R}$ and $\mathcal{N} \triangleq \mathbb{R}^\ell  \times \mathcal{Z} \subset \mathbb{R}^{\ell+1}$ be open sets. Consider the system ${\dot{\zeta} = {h}(t, \zeta)}$,  where, $\zeta \triangleq [w, z]^\top \in \mathcal{N}$, and $h : \mathbb{R}_+ \times \mathcal{N} \to \mathbb{R}^{\ell+1}$ is piecewise continuous in $t$ and locally Lipschitz in $\zeta$, uniformly in $t$, on $\mathbb{R}_+ \times \mathcal{N}$. Suppose that there exist functions $\mathcal{U}: \mathbb{R}^\ell \rightarrow \mathbb{R}_+$ and $\mathcal{V}_1: \mathcal{Z} \rightarrow \mathbb{R}_+$, continuously differentiable and positive definite in their respective domains, such that $\mathcal{V}_1(z) \to \infty$ as $|z| \to \rho$ and $\alpha_1(\|w\|) \leq \mathcal{U}(w) \leq \alpha_2(\|w\|)$, where, $\alpha_1$ and $\alpha_2$ are class $\mathcal{K}_\infty$ functions. Let $\mathcal{V}(\zeta) \triangleq \mathcal{V}_1(z) + \mathcal{U}(w)$, and $z(0) \in \mathcal{Z}$. If it holds that $\dot{\mathcal{V}} = (\nabla \mathcal{V})^\top{h} \leq 0$, in the set $z \in \mathcal{Z}$, then $z(t) \in \mathcal{Z},~\forall t \in [0, \infty)$.
\end{lemma}

\begin{figure}[t!]
	\centering{
		\subfigure[The problem]{\hspace*{-0.5cm}\includegraphics[width=4.5cm]{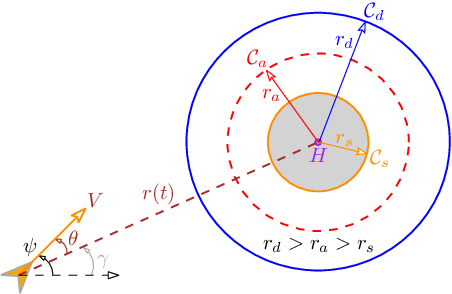}\label{fig_problem_formulation}}\hspace*{-0.1cm}
		\subfigure[Solution methodology ]{\includegraphics[width=4.0cm]{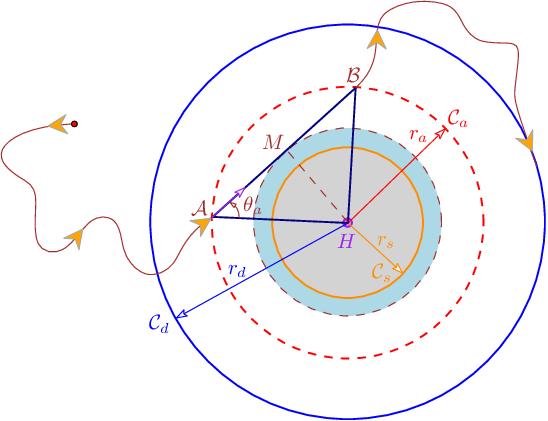}\label{fig_solution_approach}}
		\caption{Illustration of the problem and solution methodology.}
		\label{fig_problem_solution_approach}}
	\vspace*{-15pt}
\end{figure}

\section{System Description and Problem Formulation}\label{section_2}

\subsection{System Description}
We consider the nonholonomic unicycle-type robot model, moving in the $\mathbb{R}^2$ space and described using the kinematics:
\begin{equation}\label{robot_model_cartesion}
	\dot{x}(t) = V\cos\psi(t), \quad \dot{y}(t) = V\sin\psi(t), \quad \dot{\psi}(t) = \Omega,
\end{equation}
where $[x(t), y(t)]^\top \in \mathbb{R}^2$ is its position, $\psi(t) \in \mathbb[0, 2\pi)$ is the heading angle measured in the anticlockwise direction from the positive $X-$axis, $V > 0$ is the constant (forward) linear speed and $\Omega \in \mathbb{R}$ is the turn-rate controller that needs to be designed. If $\dot{\psi} \equiv \Omega_0$ is a constant, the robot moves around a circular path of radius $|V/\Omega_0|$, and if $\dot{\psi} \equiv 0$, the robot moves in a straight line. The sense of the rotation of the robot is governed by the sign of $\Omega$, which essentially controls the curvature of its trajectory \cite{jain2019trajectory}. We follow the convention that the robot moves in the anticlockwise direction if $\Omega > 0$ and it moves in the clockwise direction if $\Omega < 0$.

Let there be a hostile target $H$ situated at $[x_H, y_H]^\top \in \mathbb{R}^2$ as shown in Fig.~\ref{fig_problem_formulation}. The straight line joining the robot and target is called the line-of-sight (LoS). From Fig.~\ref{fig_problem_formulation}, it can be obtained that the LoS range $r(t) = \sqrt{(x(t)-x_H)^2 + (y(t)-y_H)^2}$ and the LoS reference angle $\gamma(t) = \tan^{-1}((y_H - y(t))/(x_H - x(t))$. We further introduce the \emph{bearing angle} $\theta(t)$ that is essentially the heading angle of the robot measured in the anticlockwise direction from the LoS line (see Fig.~\ref{fig_problem_formulation}). With these notations, the resultant \emph{robot-target engagement model} can be represented in the polar coordinates as \cite{cao2015uav}:
\begin{equation}\label{robot_model_polar}
	\dot{r}(t) = -V\cos\theta(t), \qquad \dot{\theta}(t)  = \Omega + ({V\sin\theta(t)}/{r(t)}),
\end{equation}
where $\dot{r}(t)$ is the range-rate and $\dot{\theta}(t)$ is the rate of change of bearing angle. Clearly, $\psi(t) = \theta(t) + \gamma(t)$ (see Fig.~\ref{fig_problem_formulation}). The description of various circles in Fig.~\ref{fig_problem_formulation} is given in the next subsection, followed by the problem statement.

\subsection{Problem Formulation}
As shown in Fig.~\ref{fig_problem_formulation}, the control objective is to make the robot circumnavigate the target $H$ on a \emph{desired circle} $\mathcal{C}_d$ centered at $H$ and having the (desired) radius $r_d$. In addition, it is desired that the robot maintains a prescribed safe distance $r_s$ from the hostile target $H$ at all times while stabilizing on $\mathcal{C}_d$. In other words, this requirement imposes constraints on the robot's motion such that it never enters the \emph{safety circle} $\mathcal{C}_s$ centered at $H$ and having radius $r_s$. We further consider an \emph{auxiliary circle} $\mathcal{C}_a$ with radius $r_a$ and center at $H$ (see Fig.~\ref{fig_problem_formulation}). The circle $\mathcal{C}_a$ helps in navigating the robot towards the target $H$ using range-based measurements, as discussed in detail subsequently. We further consider the following mild assumption on the three radii:

\begin{assumption}\label{assumption_radii}
In Fig.~\ref{fig_problem_formulation}, the radii $r_d$, $r_a$ and $r_s$ are such that (i) $r_d > r_a > r_s$,  (ii) $r_d < r_s + r_a$, and (iii) $r^2_a > r_d r_s$. 
\end{assumption} 

Note that the radii in Assumption~\ref{assumption_radii} form the sides of a triangle. Unless explicitly mentioned otherwise, we assume that Assumption~\ref{assumption_radii} holds throughout the paper. We now formally state the problem considered in this paper: 

\begin{problem}\label{problem}
Consider the robot model \eqref{robot_model_polar}. For the given location of the hostile target $H$ and the radii $r_d$, $r_a$ and $r_s$ satisfying Assumption~\ref{assumption_radii}, design the control $\Omega$ such that the robot circumnavigates $H$ on the desired circular path $\mathcal{C}_d$, while avoiding the safety circle $\mathcal{C}_s$ at all times, that is, $r(t) \to r_d$ and $\theta(t) \to {\pi}/{2}$, as $t \to \infty$ and $r(t) \geq r_s$ for all $t \geq 0$. 
\end{problem}

Please note that there also exists an equilibrium when $r(t) \to r_d$ and $\theta(t) \to {3\pi}/{2}$ for the robot to have a stable motion on $\mathcal{C}_d$; only the direction of rotation will be reversed. However, this case is excluded from our analysis in this paper for simplicity and clarity and can be treated similarly with some minor modifications. In the next section, we present our solution approach to Problem~\ref{problem} and propose the control $\Omega$.  

\section{Solution Approach and Control Design}\label{section_3}
 
\subsection{Solution Approach}
To solve the problem, we propose the controller of the form:
\begin{equation}\label{control_new}
	\Omega =
	\begin{cases}
		f(r, \dot{r}), & r(t) \geq r_a,\\
		0, & \text{otherwise},
	\end{cases}
\end{equation} 
where the function $f(r, \dot{r})$ needs to be designed concerning Problem~\ref{problem}. From \eqref{control_new}, it is clear that the $\Omega$ is governed by the function $f(r, \dot{r})$ if $r(t) \geq r_a$ and $\Omega = 0$ if $r(t) < r_a$, i.e., when the robot is inside the $\mathcal{C}_a$ circle, where it moves in a straight line. This situation is depicted in Fig.~\ref{fig_solution_approach} where the robot enters $\mathcal{C}_a$ at point $\mathcal{A}$ and leaves $\mathcal{C}_a$ at point $\mathcal{B}$. In this situation (i.e., when $r < r_a$), the minimum distance between the target $H$ and the robot is given by $r_{\min} = HM$ where $M$ is the midpoint of the straight line joining points the $\mathcal{A}$ and $\mathcal{B}$ such that $HM \perp \mathcal{A}\mathcal{B}$. From Fig.~\ref{fig_solution_approach}, one can easily obtain that  
\begin{equation}\label{r_min}
r_{\min} = HM = r_a\sin\theta_a(t),
\end{equation}
where $\theta_a(t)$ is the bearing angle at which the robot hits the $\mathcal{C}_a$ circle at point $\mathcal{A}$ at some instant of time $t$. Obviously, if $r_{\min} \geq r_s$, our objective of avoiding the safety circle $\mathcal{C}_s$ will be achieved. This is equivalent to the fact that the bearing angle $\theta_a(t)$ when the robot enters the $\mathcal{C}_a$ circle must satisfy the condition
\begin{equation}\label{condition_angle}
\sin\theta_a(t) \geq {r_s}/{r_a}, 
\end{equation}
for the given radii $r_a$ and $r_s$, using \eqref{r_min}. In other words, our approach relies on designing the function $f(r, \dot{r})$ in \eqref{control_new} such that the condition \eqref{condition_angle} must be satisfied for all $t \geq 0$ to prevent the robot penetrating the safety circle $\mathcal{C}_s$. Without loss of generality, we assume that the robot begins its motion when $r(t) \geq r_a$. Since if the robot begins its motion within the $\mathcal{C}_a$ circle where $r(t) < r_a$ and $\sin^{-1}({r_s}/{r(t)}) < \theta(t) < 2\pi - \sin^{-1}({r_s}/{r(t)})$, it eventually moves outside $\mathcal{C}_a$ in the finite time due to zero control input imposed inside $\mathcal{C}_a$. This finite time can be regarded as a new starting point when $r(t) \geq r_a$. Please note that the trivial case of initial bearing $\theta(t_o) \in (0, \sin^{-1}({r_s}/{r(t_o)})) \cup (2\pi - \sin^{-1}({r_s}/{r(t_o)}), 2\pi)$, where the robot begins its motion within $\mathcal{C}_a$ and directly penetrates the safety circle $\mathcal{C}_s$, is excluded from our analysis. Another important fact we shall show subsequently is that $\theta(t) \in (0, \pi)$ for all $t \geq 0$, under the proposed control scheme, which is discussed next.

\subsection{Control Design}
The proposed control design methodology involves two steps. At first, we introduce an intermediate positive scalar function characterizing the desired equilibrium of the system \eqref{robot_model_polar}. Leveraging this, we propose a candidate logarithmic BLF that solves the problem with desired safety requirements.      

\subsubsection{Equilibrium characterizing function}
For $r(t) \geq r_a$, we introduce the following scalar function:
\begin{equation}\label{eta}
	\eta(r, \theta) \triangleq 1 - \sin\theta + \phi(r), 
\end{equation}
\begin{equation}\label{phi}
\text{where}, ~ \phi(r) = \int_{r_d}^{r} [k\cos(\sin^{-1}({r_a}/{\sigma})) - ({1}/{\sigma})] d\sigma,
\end{equation}
and the control gain $k > 0$ is given by
\begin{equation}\label{gain_k}
	k = {1}/{\sqrt{r_d^2 - r_a^2}}.
\end{equation}
A discussion on how $k$ is given by \eqref{gain_k} is provided in Lemma~\ref{lem_gain_k} in the next subsection. Below, we show that \eqref{eta} is always positive and achieves its minimum value at the equilibrium $[r_d, {\pi}/{2}]^\top$, as stated in Problem~\ref{problem}. To illustrate this, we first state the following result:

\begin{lemma}\label{lem_phi}
For the function $\phi(r)$ in \eqref{phi}, it holds that $\phi(r) \geq 0$ for all $r \geq r_a$ and $\phi(r) = 0$ if $r = r_d$.
\end{lemma}

\begin{proof}
Denote $\phi(r) = \int_{r_d}^{r} g(\sigma) d\sigma$ where $g(\sigma) \triangleq (k\cos(\sin^{-1}({r_a}/{\sigma})) - ({1}/{\sigma}))$. Note that $\cos(\sin^{-1}({r_a}/{r})) = {\sqrt{r^2 - r^2_a}}/{r}$, as $0 \leq \sin^{-1}({r_a}/{r}) \leq {\pi}/{2}$ for $r \geq r_a$, one can obtain $k\cos(\sin^{-1}({r_a}/{r_d})) = {1}/{r_d}$, where we used \eqref{gain_k} for simplification. Consequently, $g(\sigma)$ can also be expressed as $g(\sigma) = [({1}/{r_d}) - ({1}/{\sigma})] + k[\cos(\sin^{-1}({r_a}/{\sigma})) - \cos(\sin^{-1}({r_a}/{r_d}))]$. Clearly, $g(r_d) = 0$. If $\sigma > r_d$, it follows that $({1}/{r_d}) - ({1}/{\sigma}) > 0$ and $\cos(\sin^{-1}({r_a}/{\sigma})) - \cos(\sin^{-1}({r_a}/{r_d})) > 0$, as $\cos(\sin^{-1}(\bullet))$ is a decreasing function in its argument $\bullet > 0$. Consequently, $\phi(r) = 0$ if $r = r_d$ and $\phi(r) > 0$ if $r > r_d$. On the other hand, if $r_a \leq \sigma < r_d$,  $({1}/{r_d}) - ({1}/{\sigma}) < 0$ and $\cos(\sin^{-1}({r_a}/{\sigma})) - \cos(\sin^{-1}({r_a}/{r_d})) < 0$. In this situation, $\phi(r) = \int_{r_d}^{r} g(\sigma) d\sigma = -\int_{r}^{r_d} g(\sigma) d\sigma > 0$, after reversal of the integration limits because $r_a \leq \sigma < r_d$. Hence, proved.   
\end{proof}
Using Lemma~\ref{lem_phi}, it can be concluded that $\eta(r, \theta) \geq 0$ and $\eta(r, \theta) = 0$ when $r(t) = r_d$ and $\theta(t) = \pi/2$, which is the desired equilibrium. Further, the value of $\eta(r, \theta)$ is the same at the two boundary points $\mathcal{A}$ and $\mathcal{B}$ in Fig.~\ref{fig_solution_approach}. It follows from the fact that (i) at the entry point $\mathcal{A}$, $r(t) = r_a$ and $\theta(t) = \theta_a$, and (ii) at the exit point $\mathcal{B}$, $r(t) = r_a$ and $\theta(t) = \pi - \theta_a$ (see Fig.~\ref{fig_solution_approach}). Now, it is straight forward to check that $\eta(r_a, \theta_a) = \eta(r_a, \pi - \theta_a)$. To facilitate the further analysis, one can evaluate the definite integral in \eqref{phi} to obtain $\phi(r) = k[\sqrt{r^2 - r^2_a} - \sqrt{r^2_d - r^2_a}] - kr_a [\tan^{-1}({\sqrt{r^2 - r^2_a}}/{r_a}) - \tan^{-1}({\sqrt{r^2_d - r^2_a}}/{r_a})] + \ln({r_d}/{r})$.

\subsubsection{BLF and Control law}
Exploiting function $\eta(r, \theta)$, we propose the following candidate logarithm BLF for $r(t) \geq r_a$:
\begin{equation}\label{W}
	W(\eta) = ({1}/{2}) \ln [{\delta^2}/(\delta^2 - \eta^2)],
\end{equation} 
where ``$\ln$" is the natural logarithm and the constant $\delta > 0$ is a design parameter, decided by the radii $r_d$, $r_a$ and $r_s$. By selecting $\delta$ appropriately, it is possible that the angle condition \eqref{condition_angle} is satisfied, thus ensuring the avoidance of the safety circle $\mathcal{C}_s$. A detailed discussion on the selection of $\delta$ is given in the next section. Note that \eqref{W} is positive-definite and $\mathcal{C}^1$ continuous in the domain where $\eta(r, \theta) < \delta$ \cite{tee2009barrier}, and $W(\eta) = 0$ if $\eta = 0$ that corresponds to the desired equilibrium. The time-derivative of \eqref{W} is obtained as $\dot{W} = ({1}/{2})[(\delta^2 - \eta^2)/{\delta^2}][{2\delta^2\eta\dot{\eta}}/(\delta^2 - \eta^2)^2] = {\eta\dot{\eta}}/(\delta^2 - \eta^2)$. Substituting $\dot{\eta}$ using \eqref{eta}, the numerator of $\dot{W}$ can be simplified as $\eta\dot{\eta} = \eta(-\dot{\theta}\cos\theta + \dot{\phi}(r))$, where $\dot{\phi}(r)$ is obtained from \eqref{phi} using Leibniz integral rule as: $\dot{\phi}(r) = \dot{r}[k\cos(\sin^{-1}({r_a}/{r})) - ({1}/{r})] = -V\cos\theta [k\cos(\sin^{-1}({r_a}/{r})) - ({1}/{r})]$, using \eqref{robot_model_polar}. Consequently, we have $\eta\dot{\eta} = \eta(-\dot{\theta}\cos\theta -V\cos\theta [k\cos(\sin^{-1}({r_a}/{r})) - ({1}/{r})]) = -\eta\cos\theta(\Omega + ({V\sin\theta}/{r}) + V [k\cos(\sin^{-1}({r_a}/{r})) - ({1}/{r})])$, using \eqref{robot_model_polar}. Now, for $r(t) \geq r_a$, choosing control
\begin{equation}\label{control}
\Omega = f(r, \dot{r}) = \frac{V}{r}(1 - \sin\theta) - k V \cos\left[\sin^{-1}\left(\frac{r_a}{r}\right)\right]  + \frac{\kappa\cos\theta}{\delta^2 - \eta^2},
\end{equation}
where $\kappa > 0$ is the control gain, yields that $\dot{W} = -\kappa {\eta \cos^2\theta}/{(\delta^2 - \eta^2)^2} \leq 0$, since $\kappa > 0$ and $\eta \geq 0$. Note that the function $f(r, \dot{r})$ in \eqref{control} can be written solely in terms of $r$ and $\dot{r}$, using \eqref{robot_model_polar}, as: $f(r, \dot{r}) = ({V}/{r})[1 - ({\sqrt{V^2 - \dot{r}^2}}/{V})] - k V \cos(\sin^{-1}({r_a}/{r}))  - ({\kappa}/{V})[{\dot{r}}/({\delta^2 - \eta^2})]$, where $\eta = 1 - ({\sqrt{V^2 - \dot{r}^2}}/{V}) + \phi(r)$. Based on the BLF \eqref{W}, we have the following result:
\begin{theorem}\label{thm_main_results}
Consider the robot-target engagement model \eqref{robot_model_polar} under the control law \eqref{control_new} where $f(r, \dot{r})$ is given by \eqref{control} for $r(t) \geq r_a$. For the given positive constant $\delta$, let the initial conditions $[r(0), \theta(0)]^\top$ belong to the set $\Theta \triangleq \{[r(t), \theta(t)]^\top \in [r_a, \infty) \times (0, \pi) \mid \eta(r, \theta) < \delta\}$, where $\eta(r, \theta)$ is defined in \eqref{eta}. Then, the following results hold:
\begin{enumerate}[leftmargin=*]
	\item[(a)] $\eta(r(t), \theta(t)) < \delta$ for all $[r(t), \theta(t)]^\top$ in the set $\Theta$.
	\item[(b)] $\dot{\eta}(r(t), \theta(t)) \leq 0$ for all $[r(t), \theta(t)]^\top$ in the set $\Theta$.
	\item[(c)] There exists a tighter bound on $\eta(r(t), \theta(t))$ given by $\eta(r(t), \theta(t)) \leq \delta\sqrt{1- {\rm e}^{-2W(\eta(0))}}$ in the set $\Theta$, where $\eta(0) \triangleq \eta(r(0), \theta(0))$ is evaluated at $t = 0$.
	\item[(d)] The control $\Omega$ is uniformly bounded by $|\Omega| \leq V (k + ({1}/{r_a})) + ({\kappa {\rm e}^{2W(\eta(0))}}/{\delta^2})$.	
\end{enumerate}
\end{theorem}

\begin{proof}
We prove each statement sequentially as follows: 
\begin{enumerate}[leftmargin=*]
\item[(a)] Since $\dot{W} \leq 0$, it now immediately follows from Lemma~\ref{lem_blf} that $\eta(r(t), \theta(t)) < \delta$ in $\Theta$. 
\item[(b)] Substituting \eqref{control} into $\eta\dot{\eta}$, $\dot{\eta}$ is obtained as
\begin{equation}\label{eta_dot}
	\dot{\eta} = -\kappa {\cos^2\theta}/({\delta^2 - \eta^2}) \leq 0,
\end{equation}
in $\Theta$, since $\eta(r(t), \delta(t)) < \delta$ in $\Theta$, as proved in part (a). 
\item[(c)] Since $\dot{W} \leq 0$, it holds that $W(\eta(r(t), \theta(t))) \leq W(\eta(0))$ in $\Theta$, where we denote by $W(0) = W(\eta(0))$. Substituting $W(\eta)$ from \eqref{W}, we have that $({1}/{2}) \ln ({\delta^2}/({\delta^2 - \eta^2})) \leq W(\eta(0))$ in $\Theta$. Taking exponential on both the sides, we have ${\delta^2}/({\delta^2 - \eta^2}) \leq {\rm e}^{2W(\eta(0))}$ in $\Theta$. Since $\delta^2 - \eta^2 > 0$, according to part (a) above, it follows that $\delta^2 - \eta^2 \geq \delta^2{\rm e}^{-2W(\eta(0))}$, which, after rearranging terms, leads to $\eta(r(t), \theta(t)) \leq \delta\sqrt{1- {\rm e}^{-2W(\eta(0))}}$ in $\Theta$.
\item[(d)] In \eqref{control}, noting that $|({1}/{r})(1 - \sin\theta)| \leq {1}/{r_a}$ since $r(t) \geq r_a$, $|-\cos(\sin^{-1}({r_a}/{r}))|$ is bounded by 1, and $|{\cos\theta}/({\delta^2 - \eta^2})| \leq {{\rm e}^{2W(\eta(0))}}/{\delta^2}$, using the above part (c). Now, the required result immediately follows, since $\Omega = 0$ for $r(t) < r_a$.  
\end{enumerate}
This completes the proof. 
\end{proof}
 
We are now ready to answer why the control gain $k > 0$ is given by \eqref{gain_k} in the following lemma:

\begin{lemma}\label{lem_gain_k}
Consider the robot model \eqref{robot_model_polar} under the control law \eqref{control_new} where $f(r, \dot{r})$ is given by \eqref{control}. If a stable circular motion exists with equilibrium $[r_d, {\pi}/{2}]^\top$, the control gain $k$ is given by \eqref{gain_k}. Further, the robot moves in the clockwise direction on the desired circle $\mathcal{C}_d$. 
\end{lemma}

\begin{proof}
Let a stable circular motion exist with equilibrium $[r_d, {\pi}/{2}]^\top$. Since the robot moves with the constant speed $V$, the magnitude of the turn-rate of the robot must be equal to ${V}/{r_d}$ on the circle $\mathcal{C}_d$. Note that the robot cannot stay within $\mathcal{C}_a$ due to zero control, it follows from \eqref{control} that the magnitude $|\Omega|$ of the turn-rate $\Omega$ must be equal to ${V}/{r_d}$ at the equilibrium $[r_d, {\pi}/{2}]^\top$. In other words, $|\Omega|_{[r_d, {\pi}/{2}]^\top} = |- k V \cos(\sin^{-1}({r_a}/{r_d}))|$ must be equal to ${V}/{r_d}$. On equating both, we get $k$ as in \eqref{gain_k}. Moreover, it can be verified from \eqref{control} that if $r = r_d$, $\Omega < 0$, i.e., the robot moves in the clockwise direction on the circle $\mathcal{C}_d$. This completes the proof.  
\end{proof}

\begin{remark}
One of the closely related works in this direction, but without consideration of the safety aspect from the hostile target $H$, is addressed in \cite{cao2015uav}, where the control $\Omega$ is given by 
\begin{equation}\label{control_old}
\hspace*{-0.3cm}\Omega =
	\begin{cases}
		k\left[-V\cos\left(\sin^{-1}\left(\frac{r_a}{r(t)}\right)\right) - \dot{r}(t)\right], & r(t) \geq r_a,\\
		0, & \text{otherwise}.
	\end{cases}
\end{equation} 
Here, the control gain $k$ is the same as \eqref{gain_k}. Depending on $k$, it may be possible that the robot does not even enter the $\mathcal{C}_a$ circle, and hence, entering into the safety circle $\mathcal{C}_s$ is avoided. However, this is possible only for the sufficiently large $k$. In some cases, entering into $\mathcal{C}_s$ may also be avoided by increasing the radius $r_a$ of $\mathcal{C}_a$ in the sense that if the boundary of $\mathcal{C}_a$ is sufficiently far from the hostile target $H$, it may happen that the robot does not move close to the safety circle $\mathcal{C}_s$. Nevertheless, this depends on the robot's initial condition and is just a simplified observation that might work in some special cases. Further, it is worth noticing from \eqref{gain_k} that the gain $k$ is large for large $r_a$, and $k \to \infty$ if $r_a \to r_d$. In practice, we cannot increase $k$ beyond a certain limit, since large $k$ means applying the high force on the robot as the control \eqref{control_old} is proportional to the gain $k$, which is undesirable due to physical constraints of a robot. In this work, we present an approach that alleviates such a requirement of large gain $k$ and ensures avoiding the safety circle $\mathcal{C}_s$ at all times.   
\end{remark}

\section{Selection of $\delta$ and Stability Analysis}\label{section_4}
In this section, we propose a selection criterion for the design parameter $\delta$ in \eqref{W} such that the condition \eqref{condition_angle} is guaranteed. This is followed by the stability analysis of the desired equilibrium, given such a choice of $\delta$. 

\subsection{Selection of $\delta$}

\begin{theorem}\label{thm_bound_delta}
For the given radii $r_d$, $r_a$ and $r_s$ satisfying Assumption~\ref{assumption_radii} and under the conditions given in Theorem~\ref{thm_main_results}, the inequality \eqref{condition_angle} holds if $\delta$ is chosen to satisfy:
\begin{equation}\label{bound_delta}
\delta \leq ({1}/{\beta})\tan^{-1}(\beta) + \ln({r_d}/{r_a}) - ({r_s}/{r_a}),
\end{equation}
where $\beta = \sqrt{r_d^2 - r_a^2}/{r_a}$. 
\end{theorem}

Before proving Theorem~\ref{thm_bound_delta}, we first show that the RHS of \eqref{bound_delta} is always positive, which is required to fulfill the condition $\eta(r, \theta) < \delta$ in Theorem~\ref{thm_main_results}, since $\eta$ is always positive. 

\begin{lemma}\label{lem_bound_Delta}
Denote $\Delta \triangleq ({1}/{\beta})\tan^{-1}(\beta) + \ln({r_d}/{r_a}) - ({r_s}/{r_a})$ and assume that the radii $r_d$, $r_a$ and $r_s$ satisfy Assumption~\ref{assumption_radii}. Then, it holds that $0 < \Delta < 1$.   
\end{lemma}

\begin{proof}
We first prove that $\Delta < 1$. Utilizing the upper bounds of the inequalities \eqref{shafer_inequality} and \eqref{logarithmic_inequality}, it can be observed that ${\tan^{-1}(\beta)}/{\beta} < 1$ because $\beta > 0$, and $\ln({r_d}/{r_a}) < ({r_d}/{r_a}) - 1$. Therefore, it follows that $\Delta < 1 + [({r_d}/{r_a}) - 1] - ({r_s}/{r_a}) = ({r_d - r_s})/{r_a}$. Since $r_d < r_a + r_s$ (Assumption~\ref{assumption_radii}), the preceding inequality implies that $\Delta < 1$. To prove that $\Delta > 0$, we use the lower bounds of the inequalities \eqref{shafer_inequality} and \eqref{logarithmic_inequality} to obtain $\Delta > {3}/({1 + 2\sqrt{1 + \beta^2}}) + [1 - ({r_a}/{r_d})] - ({r_s}/{r_a})$. Substituting for $\beta$, $\sqrt{1 + \beta^2} = \sqrt{1 + ((r_d^2 - r_a^2)/{r_a^2})} = {r_d}/{r_a}$. Further, since $r_a^2 > r_d r_s$ (Assumption~\ref{assumption_radii}), it follows that ${r_s}/{r_a} < {r_a}/{r_d} \implies -{r_s}/{r_a} > - {r_a}/{r_d}$. Exploiting these, it can be written that $\Delta > {3r_a}/(r_a + 2r_d) + [1 - ({r_a}/{r_d})] - ({r_a}/{r_d}) = {3r_a}/(r_a + 2r_d) + 1 - ({2r_a}/{r_d}) = (3r_a r_d + r_d(r_a + 2 r_d) - 2r_a(r_a + 2r_d))/(r_d(r_a + 2r_d)) = {2(r_d^2 - r_a^2)}/(r_d(r_a + 2r_d)) > 0$, since $r_d > r_a$. On combining, it implies $0 < \Delta < 1$. 
\end{proof}

Clearly, it follows from Lemma~\ref{lem_bound_Delta} and Theorem~\ref{thm_main_results} that $\delta$ in \eqref{bound_delta} satisfies $0 < \delta < 1$. We now proceed to prove Theorem~\ref{thm_bound_delta}. 

\begin{proof}[Proof of Theorem~\ref{thm_bound_delta}]
According to Theorem~\ref{thm_main_results}, since the condition $\eta(r, \theta) < \delta$ is satisfied in $\Theta$, it must also hold at the entry point $\mathcal{A}$ in Fig.~\ref{fig_solution_approach} where $r(t) = r_a$ and $\theta(t) = \theta_a$. Consequently, it must be satisfied that $\eta(r_a, \theta_a) < \delta$. Note that $\eta(r_a,\theta_a) = 1 - \sin(\theta_a) + \phi(r_a) \iff \sin\theta_a = ({1}/{\beta}) \tan^{-1}(\beta) + \ln(r_d/r_a) - \eta(r_a,\theta_a) > ({1}/{\beta}) \tan^{-1}(\beta) + \ln(r_d/r_a) - \delta$. Now, using \eqref{bound_delta}, the preceding inequality implies that $\sin\theta_a > ({1}/{\beta}) \tan^{-1}(\beta) + \ln(r_d/r_a) - ({1}/{\beta}) \tan^{-1}(\beta) - \ln(r_d/r_a) + ({r_s}/{r_a})  = {r_s}/{r_a}$. 
\end{proof}

Next, we state one of the crucial results about the bearing angle $\theta(t)$ in the following theorem. 

\begin{theorem}\label{thm_range_theta}
Under the conditions given in Theorem~\ref{thm_main_results} with $\delta$ chosen according to Theorem~\ref{thm_bound_delta}, it holds that $\theta(t) \in (0, \pi)$ for all $t \geq 0$. 
\end{theorem}

\begin{proof}
We consider the following two cases:
\begin{itemize}[leftmargin=*]
	\item If $r(t) \geq r_a$. In this situation, it follows from Theorem~\ref{thm_main_results} that $1 - \sin\theta + \phi(r) < \delta$ where $\delta$ is given by \eqref{bound_delta}. Since $0 < \delta < 1$, it follows that $\sin\theta$ can never be non-positive otherwise $\eta(r, \theta) \geq 1$, irrespective of the value of $\phi(r)$ (which is always positive, see Lemma~\ref{lem_phi}). This contradicts the condition $\eta(r, \theta) < \delta$. Hence, $\theta(t) \in (0, \pi)$ if $r(t) \geq r_a$.
	\item When $r(t) < r_a$. In this situation, the robot moves inside $\mathcal{C}_a$. Since the control $\Omega$ is zero inside $\mathcal{C}_a$ and the robot enters $\mathcal{C}_a$ at an angle $\theta_a$, which satisfies the condition \eqref{condition_angle} for the choice of $\delta$ in \eqref{bound_delta}, it holds that $\theta(t) \in (\sin^{-1}({r_s}/{r_a}), \pi - \sin^{-1}({r_s}/{r_a})) \subset (0, \pi)$ for $r(t) < r_a$.   
\end{itemize}
Combining, it can be concluded that $\theta(t) \in (0, \pi), \forall t \geq 0$. 
\end{proof}

\subsection{Stability Analysis}
Building upon the above results, we now prove the stability of the closed-loop system for the choice of $\delta$ in \eqref{bound_delta} in the following theorem:
 
\begin{theorem}\label{thm_convergence_analysis}
Consider the robot-target engagement model \eqref{robot_model_polar} under the control law \eqref{control_new} where $f(r, \dot{r})$ is given by \eqref{control}. Let the initial conditions $[r(0), \theta(0)]^\top$ belong to the set $\Theta = \{[r(t), \theta(t)]^\top \in [r_a, \infty) \times (0, \pi) \mid \eta(r, \theta) < \delta\}$, where $\delta$ is chosen according to \eqref{bound_delta} with radii $r_d$, $r_a$ and $r_s$ subject to Assumption~\ref{assumption_radii}. If $\kappa > 0$ and $k$ is given by \eqref{gain_k}, then $r(t) \to r_d$ and $\theta(t) \to {\pi}/{2}$ as $t \to \infty$ and $r(t) \geq r_s$ for all $t \geq 0$.
\end{theorem}

\begin{proof}
Without loss of generality, we assume that the robot enters $\mathcal{C}_a$ at point $\mathcal{A}$ at time instant $t_e$ and exits it at point $\mathcal{B}$ at the first subsequent time $t_x$\footnote{If the robot starts within $\mathcal{C}_a$ it eventually comes out of $\mathcal{C}_a$ due to $\Omega = 0$.}. The time interval $t_x - t_e$ is bounded (by ${2r_a}/{V}$) since the robot moves in the straight line inside $\mathcal{C}_a$. Depending upon the gain $\kappa$ in \eqref{control}, the robot may enter $\mathcal{C}_a$ multiple times (a detailed discussion on this is given in the next section), and the number of time instances $t_e$ is the same as $t_x$ since the robot cannot stay inside $\mathcal{C}_a$ as the control $\Omega = 0$ inside $\mathcal{C}_a$. As soon as the robot is outside $\mathcal{C}_a$ (i.e., $r(t) \geq r_a$), we now divide the proof into the two steps when the period $[t_e, t_x]$ is dropped: (i) At first, we show that the robot enters $\mathcal{C}_a$ circle only finite number of times, (ii) once it remains outside $\mathcal{C}_a$, the robot finally reach to the desired equilibrium.\par
\emph{Step 1:} Since $0 \leq \cos^2\theta \leq 1$ and ${1}/(\delta^2 - \eta^2) \leq {{\rm e}^{2W(0)}}/{\delta^2}$ almost everywhere when $[t_e, t_x]$ is excluded (see Theorem~\ref{thm_main_results}), it follows from \eqref{eta_dot} that $\dot{\eta} \leq -\mu$ where $\mu = {\kappa {\rm e}^{2W(0)}}/{\delta^2}$ is a positive constant. Since $\eta$ is bounded by $\delta$ (see Theorem~\ref{thm_main_results}) and has bounded derivative $\dot{\eta} \leq -\mu$, $\eta$ is uniformly continuous almost everywhere when $[t_e, t_x]$ is excluded. Integrating $\dot{\eta} \leq -\mu$ gives $\eta(r(t), \theta(t)) \leq \eta((0), \theta(0)) - \mu t$, implying that there exists a time instant $t^\star$ such that $\eta \leq \epsilon$ for some sufficiently small $\epsilon > 0$, when $[t_e, t_x]$ is excluded. In other words, there exists a time instant $t^\star$ such that the robot remains in the vicinity of the equilibrium $[r_d, {\pi}/{2}]^\top$ when $[t_e, t_x]$ is excluded, since $\Xi \triangleq \{[r(t), \theta(t)]^\top \in [r_a, \infty) \times (0, \pi) \mid \eta \leq \epsilon\}$ is a compact and positively invariant set (as $\eta$ is positive-definite and continuously differentiable, and satisfies $\dot{\eta} \leq 0$). Alternatively, this implies that there exists a time instant $t^\star$ such that $r(t^\star) \geq r_a$ for all $t \geq t^\star$ when $[t_e, t_x]$ is excluded, that is, the robot enters $\mathcal{C}_a$ only finite number of times and remains outside $\mathcal{C}_a$ for all $t \geq t^\star$. \par
\emph{Step~2:} For $t \geq t^\star$, we consider the BLF \eqref{W} whose time-derivative $\dot{W}$ is non-positive. Using LaSalle's invariance theorem \cite[Theorem~4.4]{khalil2002control}, it can be concluded that all the trajectories of \eqref{robot_model_polar} converge to the largest invariant set $\Gamma$ contained in the set $\Lambda \triangleq \{[r, \theta]^\top \in [r_a, \infty) \times (0, \pi) \mid \dot{W} = 0\}$. Note that $\dot{W} = 0$ if $\eta = 0$ or $\cos\theta = 0 \implies \theta = \pi/2$ in the set $\Lambda$. Clearly, $\eta = 0$ corresponds to the desired equilibrium $[r, \theta]^\top = [r_d, {\pi}/{2}]^\top$. On the other hand, $\theta = \pi/2$ corresponds to the multiple equilibrium points $[r, {\pi}/{2}]^\top$ in $\Lambda$ where $r \geq r_a$. We now show that $\theta = \pi/2$ holds only if $r = r_d$, and hence, the largest invariant set $\Gamma \subset \Lambda$ is indeed the desired equilibrium $[r_d, {\pi}/{2}]^\top$. If $\theta = \pi/2 \implies \dot{\theta} = 0$ in $\Lambda$. In this situation, it follows from \eqref{robot_model_polar} that $\dot{r} = 0$, i.e., $r$ is constant and $0 = ({V}/{r}) - k V \cos(\sin^{-1}({r_a}/{r})) \implies k = 1/\sqrt{r^2 - r_a^2}$, which is valid only if $r = r_d$ (see \eqref{gain_k}). Consequently, it follows that the robot asymptotically converges to the desired equilibrium, i.e., $[r(t), \theta(t)]^\top \to [r_d, {\pi}/{2}]^\top$, as $t \to \infty$.\par 
Step~1 and Step~2 jointly prove the first part of the theorem. Further, since it holds that the entry angle $\theta(t_e) = \theta_a \geq \sin^{-1}({r_s}/{r_a})$ at each $t_e$ because of the choice of $\delta$ in \eqref{bound_delta}, it is clear from Fig.~\ref{fig_solution_approach} that $r(t) \geq r_s$ for each interval $[t_e, t_x]$ when $r(t) < r_a$. Therefore, $r(t) \geq r_s, \forall t \geq 0$.  
\end{proof}

It is possible to show the exponential convergence to the desired equilibrium for the linearized closed-loop system \eqref{robot_model_polar}, under the control \eqref{control_new}. This implies that the convergence rate is bounded by some known rate near the equilibrium point, which shows robustness against disturbances in its vicinity \cite{sastry2011adaptive}. We have the following result in this direction: 

\begin{theorem}\label{thm_exponential_stability}
Consider the closed-loop system \eqref{robot_model_polar} under the control \eqref{control_new} where $f(r, \dot{r})$ is given by \eqref{control}. Then, $[r_d, {\pi}/{2}]^\top$ is a locally exponentially stable equilibrium.   
\end{theorem}

\begin{proof}
Let us denote $f_r(r(t), \theta(t)) \triangleq \dot{r} =  -V\cos\theta$ and $f_{\theta}(r(t), \theta(t)) \triangleq \dot{\theta} = ({V}/{r}) - k V \cos(\sin^{-1}({r_a}/{r})) + {\kappa\cos\theta}/(\delta^2 - \eta^2)$. The linearization about the equilibrium $[r_d, {\pi}/{2}]^\top$ leads to the system $\dot{\xi} = A\xi$ where $\xi \triangleq [r(t), \theta(t)]^\top$ and $A = [A_{ij}] \in \mathbb{R}^{2 \times 2}$ with $A_{11} = {\partial f_r}/{\partial r}{|}_{[r_d, {\pi}/{2}]^\top} = 0, \ A_{12} = {\partial f_r}/{\partial \theta}{|}_{[r_d, {\pi}/{2}]^\top} = V, \ A_{21} = {\partial f_{\theta}}/{\partial r}{|}_{[r_d, {\pi}/{2}]^\top} = -k^2V, \ A_{22} = {\partial f_{\theta}}/{\partial \theta}{|}_{[r_d, {\pi}/{2}]^\top} = -{\kappa}/{\delta^2}$, where we obtained $A_{21}$ and $A_{22}$ at $[r_d, {\pi}/{2}]^\top$ using $A_{21} = -({V}/{r^2}) - {k V r_a^2}/(r^2\sqrt{r^2 - r_a^2}) + [{2\kappa \eta \cos\theta}/{(\delta^2 - \eta^2)^2}]({\partial \eta}/{\partial r}),\ A_{22} = -{\kappa \sin\theta}/(\delta^2 - \eta^2) + [{2\kappa \eta \cos\theta}{(\delta^2 - \eta^2)^2}]({\partial \eta}/{\partial \theta})$. The eigenvalues of matrix $A$ are given by $\lambda = (-\kappa \pm \sqrt{\kappa^2 - 4 k^2 \delta^2 V^2})/{2\delta^2}$. Clearly, $A$ is Hurwitz, irrespective of the sign of the term $\kappa^2 - 4 k^2 \delta^2 V^2$. If $\kappa^2 - 4 k^2 \delta^2 V^2 < 0$, the eigenvalues are complex conjugate with negative real part $-\kappa/2\delta^2$. If $\kappa^2 - 4 k^2 \delta^2 V^2 \geq 0$, both the eigenvalues are real and negative as $\sqrt{\kappa^2 - 4 k^2 \delta^2 V^2} < \kappa$. Now, the exponential convergence follows using \cite[Theorem~4.7]{khalil2002control}. This concludes the proof. 
\end{proof}

\section{Robot Motion Inside Auxiliary Circle}\label{section_5}
Note that the choice of the gain $\kappa > 0$ is flexible in \eqref{control}, as opposed to the gain $k$, which is fixed according to \eqref{gain_k} for the given radii. This section analyzes the (sufficient) condition on $\kappa$ characterizing the number of entries of the robot inside the $\mathcal{C}_a$ circle before stabilizing to the desired circle $\mathcal{C}_d$. 
  
\begin{theorem}\label{thm_kappa_conditions}
Consider the robot-target engagement model \eqref{robot_model_polar} subject to the control law \eqref{control_new} where $f(r, \dot{r})$ is given by \eqref{control}. Under the conditions given in Theorem~\ref{thm_main_results}, if the gain $\kappa$ in \eqref{control} is chosen such that  $\kappa \geq kV\delta^2{\rm e}^{-2W(0)}$ (resp., $0 < \kappa < kV\delta^2{\rm e}^{-2W(0)}$), then the robot can move inside the auxiliary circle $\mathcal{C}_a$ at most once (resp., more than once). 
\end{theorem}

\begin{figure}[t]
	\centering{
		\includegraphics[width=3.0cm]{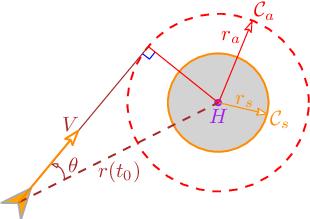}
		\caption{Geometry when robot's velocity vector is tangent to $\mathcal{C}_a$ at time instant $t_0$.}	
	\label{fig_tangent}}
	\vspace*{-15pt}
\end{figure}

Before proving Theorem~\ref{thm_kappa_conditions}, we first state the following lemma: 
\begin{lemma}\label{lem_primary_result_kappa}
Consider the robot-target engagement model \eqref{robot_model_polar} subject to the control law \eqref{control_new}. Let there be a time instant $t_0$ such that $r(t_0) \geq r_a$ and $\theta(t_0) \in (\theta_0, \pi)$ where $\theta_0 = \sin^{-1}({r_a}/{r(t_0)})$ as shown in Fig.~\ref{fig_tangent}. If the gain $\kappa$ in \eqref{control} is chosen such that $\kappa \geq kV\delta^2{\rm e}^{-2W(0)}$, then $r(t) \geq r_a, \forall t \geq t_0$. 
\end{lemma}

Although a proof of Lemma~\ref{lem_primary_result_kappa} follows from Step 1 of the proof of Theorem~\ref{thm_convergence_analysis} by replacing $t_0$ by $t^\star$, we provide below geometrical reasoning for the same that will be helpful to prove Theorem~\ref{thm_kappa_conditions}. 

\begin{proof}
	The proof contains the following two steps:\par
\emph{Step~1:} We first prove that $\theta(t) \in [\sin^{-1}({r_a}/{r(t)}), \pi)$ for all $t \in [t_0, t_0+T]$, if $r(t) \geq r_a$ for all $t \in [t_0, t_0+T]$, where $T$ is any positive constant. If $\theta(t) = \sin^{-1}({r_a}/{r(t)})$, it follows that $\dot{\theta}(t) = ({V}/{r}) + [({\kappa}/(\delta^2 - \eta^2)) - kV]({\sqrt{r^2 - r_a^2}}/{r})$. For $\kappa \geq kV\delta^2{\rm e}^{-2W(0)}$, $\dot{\theta}(t) \geq ({V}/{r}) + kV[{\delta^2{\rm e}^{-2W(0)}}(\delta^2 - \eta^2) - 1]({\sqrt{r^2 - r_a^2}}/{r})$. Since ${\delta^2}/(\delta^2 - \eta^2) \leq {\rm e}^{2W(\eta(0))}$, according to the part (c) of Theorem~\ref{thm_main_results}, it follows that $\dot{\theta}(t) \geq {V}/{r} > 0$, i.e., $\theta(t)$ cannot be smaller than  $\sin^{-1}({r_a}/{r(t)})$ for all $t \in [t_0, t_0+T]$. Hence, this step is proved in conjunction with Theorem~\ref{thm_range_theta}. \par 
\emph{Step~2:} We now prove that $r(t) \geq r_a$ for all $t \in [t_0, t_0+T]$, if $\theta(t) \in [\sin^{-1}({r_a}/{r(t)}), \pi)$ for all $t \in [t_0, t_0+T]$, where $T$ is any positive constant. When $r(t) = r_a$, $\sin^{-1}({r_a}/{r(t)}) = \pi/2$. Then, $\theta(t) \in [\sin^{-1}({r_a}/{r(t)}), \pi)$ implies $\theta(t) \in [{\pi}/{2}, \pi)$, for which, $\dot{r} \geq 0$ (see \eqref{robot_model_polar}), i.e., $r(t)$ can not be smaller than $r_a$. \par 
By combining both steps, we can see that if one claim holds in a step, the other also holds. When both claims are true at time $t_0$, they remain true until one fails. Because of the interdependency between the claims, it is impossible for them to fail simultaneously or for one to fail before the other. Hence, both claims hold for all time.
\end{proof}

\begin{proof}[Proof of Theorem~\ref{thm_kappa_conditions}]
When the robot enters $\mathcal{C}_a$ at time instant $t_e$, where $r(t_e) = r_a$ and $\sin\theta(t_e) = \sin\theta_a \geq {r_s}/{r_a}$ (see \eqref{condition_angle}), it follows from Step 2 of the proof of Lemma~\ref{lem_primary_result_kappa} that $\theta(t_e) \in [\sin^{-1}({r_s}/{r_a}), {\pi}/{2})$. Since $\Omega = 0$ inside $\mathcal{C}_a$, the robot exits $\mathcal{C}_a$ at some time $t_x$ where it holds that $\theta(t_x) = \pi - \theta_a$ (see Fig.~\ref{fig_solution_approach}). As a consequence, $\theta(t_x) \in ({\pi}/{2}, \pi - \sin^{-1}({r_s}/{r_a})]$. Since $\sin^{-1}({r_a}/{r(t_x)}) = {\pi}/{2}$, it can be rewritten that $\theta(t_x) \in (\sin^{-1}({r_a}/{r(t_x)}), \pi - \sin^{-1}({r_s}/{r_a})]$ when $r(t_x) = r_a$. Clearly, $(\sin^{-1}({r_a}/{r(t_x)}), \pi - \sin^{-1}({r_s}/{r_a})] \subset (\sin^{-1}({r_a}/{r(t_x)}), \pi)$. Now, by considering $t_x = t_0$ as in Lemma~\ref{lem_primary_result_kappa}, it immediately follows that the robot will never move inside $\mathcal{C}_a$. In other words, it can move inside $\mathcal{C}_a$ at most once if  $\kappa \geq kV\delta^2{\rm e}^{-2W(0)}$. On the other hand, if $0 < \kappa < kV\delta^2{\rm e}^{-2W(0)}$, it can be verified, using the part (c) of Theorem~\ref{thm_main_results}, that the term ${\kappa}/(\delta^2 - \eta^2) - kV < 0$. As a consequence, $\dot{\theta}(t)$ may assume negative values whenever $V < \sqrt{r^2 - r_a^2}$. This implies that Step 1 of the proof of Lemma~\ref{lem_primary_result_kappa} is no longer valid, and hence, Step 2 because of the dependency of the two claims. This indicates that the robot might enter inside $\mathcal{C}_a$ more than once if $0 < \kappa < kV\delta^2{\rm e}^{-2W(0)}$. However, the number of entries to $\mathcal{C}_a$ are finite, as proved in the Step 1 of Theorem~\ref{thm_convergence_analysis}. 
\end{proof}

\begin{figure}[t!]
	\centering{
		\subfigure[Single entry in $\mathcal{C}_a$]{\hspace*{-0.2cm}\includegraphics[width=3.3cm]{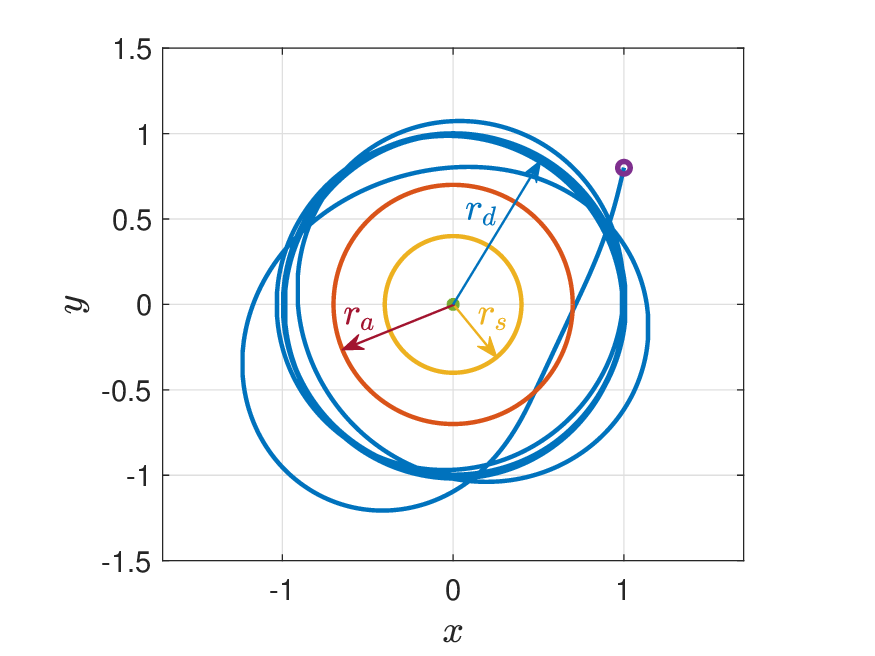}\label{single_entry}}\hspace*{-0.2cm}
		\subfigure[Multiple entries in $\mathcal{C}_a$ ]{\includegraphics[width=3.3cm]{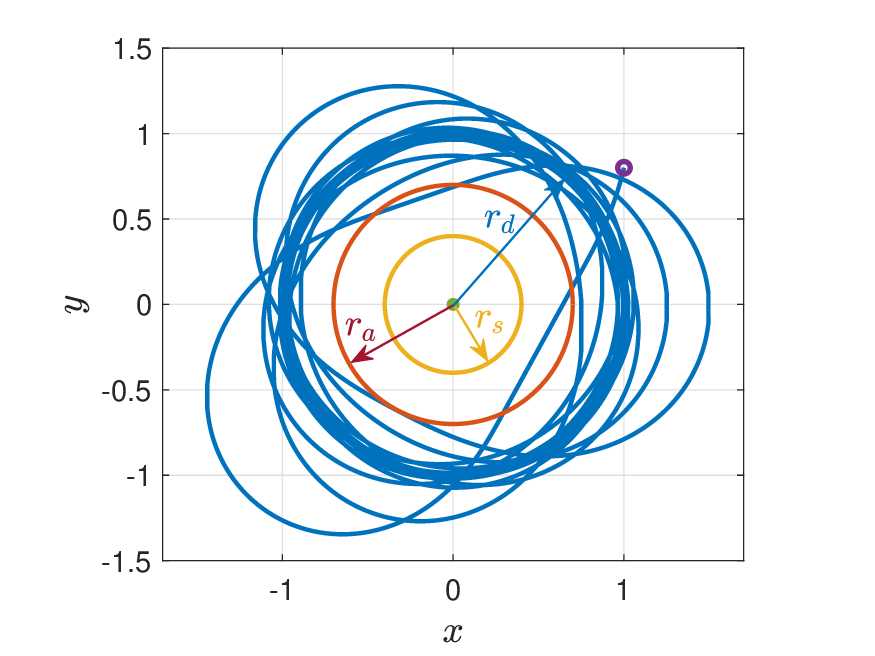}\label{multi_entry}}
		\caption{Trajectory of the robot when it enters inside $\mathcal{C}_a$ (a) only one time with $\kappa = 0.05$ (b) multiple times with $\kappa = 0.015$.}
		\label{fig_trajectory}}
		\vspace*{-15pt}
\end{figure}

\begin{figure}[t!]
	\centering{
		\subfigure[Range]{\hspace*{-0.3cm}\includegraphics[width=3.3cm]{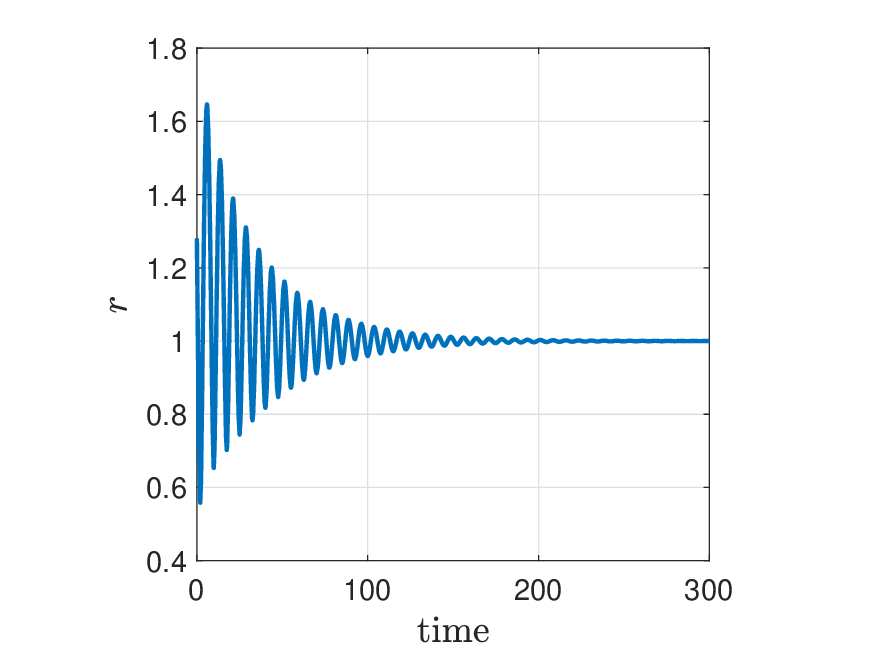}\label{range}}\hspace*{-0.53cm}
		\subfigure[Bearing]{\includegraphics[width=3.3cm]{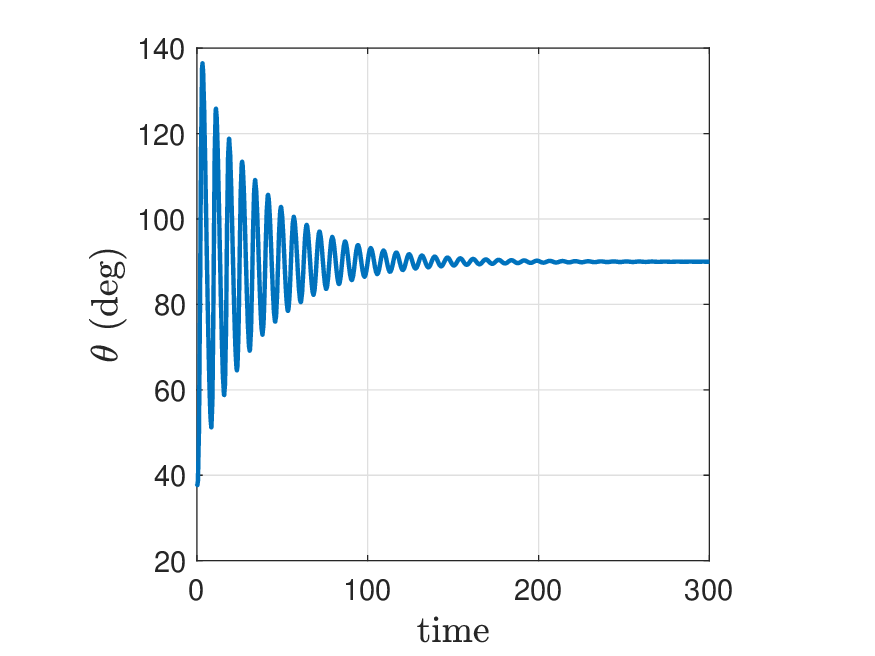}\label{bearing}}\hspace*{-0.53cm}
		\subfigure[Control]{\includegraphics[width=3.3cm]{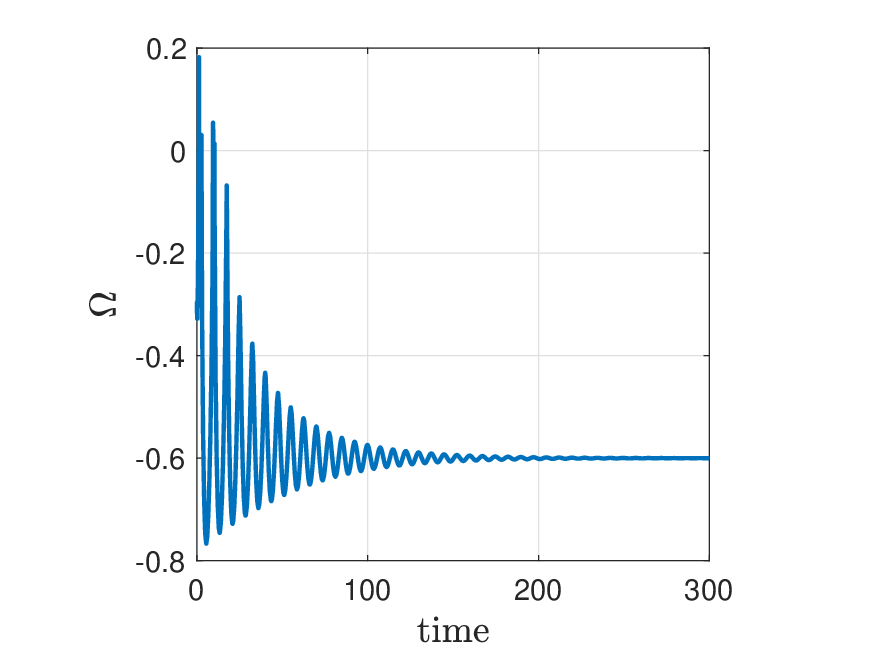}\label{control_law}}	
		\caption{Evolution of range $r(t)$, bearing angle $\theta(t)$ and control law $\Omega$ with time for multiple entry case in Fig.~\ref{multi_entry}.}
		\label{fig_paremetrs}}
		\vspace*{-15pt}
\end{figure}

\section{Simulation and Experimental Results}\label{section_6}
Consider the robot begins with states $[x(0), y(0), \theta(0)]^\top = [1~\text{m}, 0.8~\text{m}, 38^\circ]^\top$ and moves with $V = 0.6$ m/s. The target is situated at $[x_H, y_H]^\top = [0, 0]^\top$. The radii of the circles $\mathcal{C}_d$, $\mathcal{C}_a$ and $\mathcal{C}_s$ are chosen as $r_d = 1$ m, $r_a = 0.7$ m, and $r_s = 0.4$ m, respectively, satisfying Assumption~\ref{assumption_radii}. Please note that these initial conditions are chosen according to the available space for the experiments at our laboratory, and we use the same setting for the simulation and experiments for better illustration. For the given radii, we choose $\delta = 0.5 < \Delta = 0.5649$, in accordance with Theorem~\ref{thm_bound_delta}. Note that the gain $k = 1.4$, using \eqref{gain_k}. From \eqref{eta} and \eqref{W}, one can find that $\eta(r(0), \theta(0)) = 0.4454$ and $W(\eta(0)) = 0.7891$.  

\subsection{Simulation Results}
Choosing different values of $\kappa$, we obtain the trajectory plots as shown in Fig.~\ref{fig_trajectory}. Fig.~\ref{single_entry} is obtained for $\kappa = 0.05$ where the robot enters $\mathcal{C}_a$ at most once following the fact that $\kappa (= 0.05) > kV\delta^2{\rm e}^{-2W(0)} (= 0.0433)$ (see Theorem~\ref{thm_kappa_conditions}). While Fig.~\ref{multi_entry} is obtained for $\kappa = 0.015$ where the robot enters $\mathcal{C}_a$ multiple times, since $\kappa (= 0.015) < kV\delta^2{\rm e}^{-2W(0)} (= 0.0433)$. In summary, if the goal is for the robot to assess the hostile target from a closer proximity to gather crucial data, a smaller value of $\kappa$ may be chosen, allowing for multiple entries into $\mathcal{C}_a$ before settling down on the desired circle $\mathcal{C}_d$. However, if the objective is to achieve faster convergence to the desired circle $\mathcal{C}_d$, a larger value of $\kappa$ may be selected, provided that the robot's physical constraints are satisfied. Further, Fig.~\ref{fig_paremetrs} illustrates the variation of range $r(t)$, bearing angle $\theta(t)$ and control $\Omega$ in \eqref{control_new} with time only for the multiple entry case in Fig.~\ref{multi_entry}. Clearly, (i) the robot stabilizes on the desired circle $\mathcal{C}_d$ and $r(t) > r_s = 0.4, \forall t \geq 0$ (see Fig.~\ref{range}), (ii) $\theta(t)$ remains between $(0, \pi), \forall t \geq 0$ (as in Theorem~\ref{thm_range_theta}) and converges at $90^\circ$ (see Fig.~\ref{bearing}), (iii) $\Omega$ converges to the desired value $-(V/r_d) = -0.6$, where ``$-$" sign indicates that the robot moves in the clockwise direction.

\begin{figure}[t!] 	
	\centering{
		\hspace*{-0.1cm}
		\subfigure[Complete trajectory on track manager]{\includegraphics[width=0.14\textwidth]{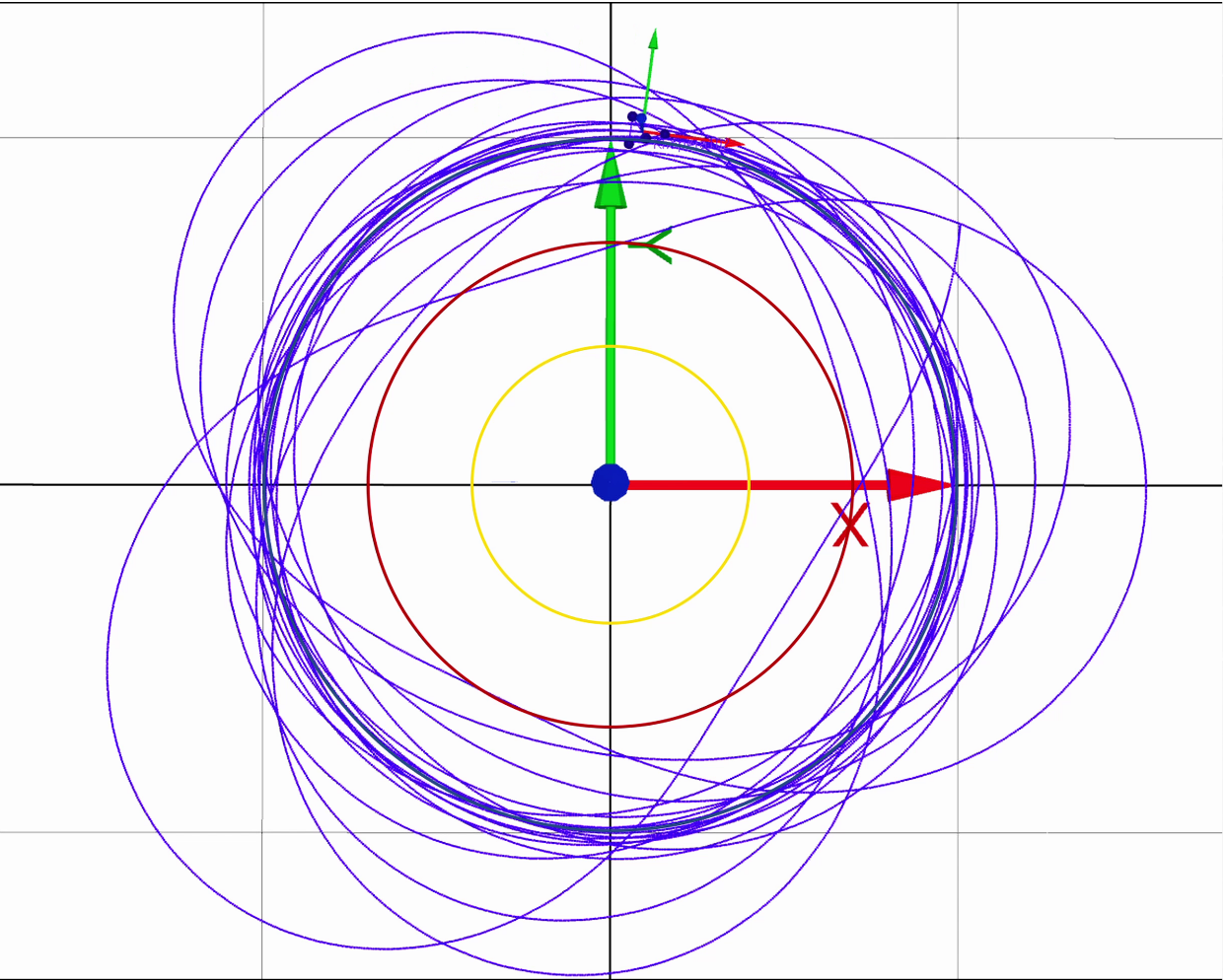}\label{qtm_multi_final}}
		\hspace*{0.15cm} 	 	
		\subfigure[Entering $C_a$ first time at $t=5 s$]{\includegraphics[width=0.14\textwidth]{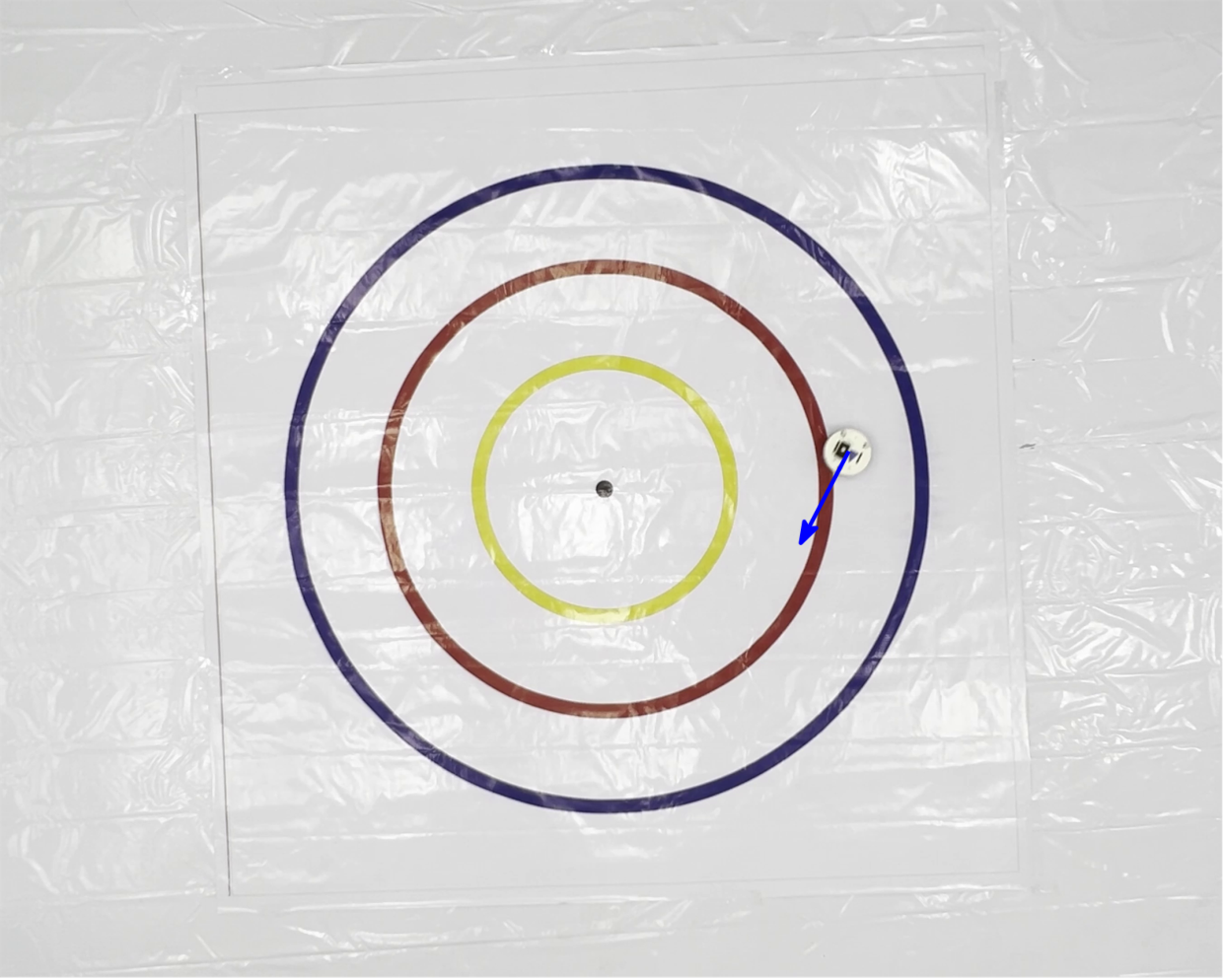}\label{multi_enter}} 	 
		 \hspace*{0.15cm}
		\subfigure[Circumnavigating $H$ on $C_d$ at $t=144 s$]{\includegraphics[width=0.14\textwidth]{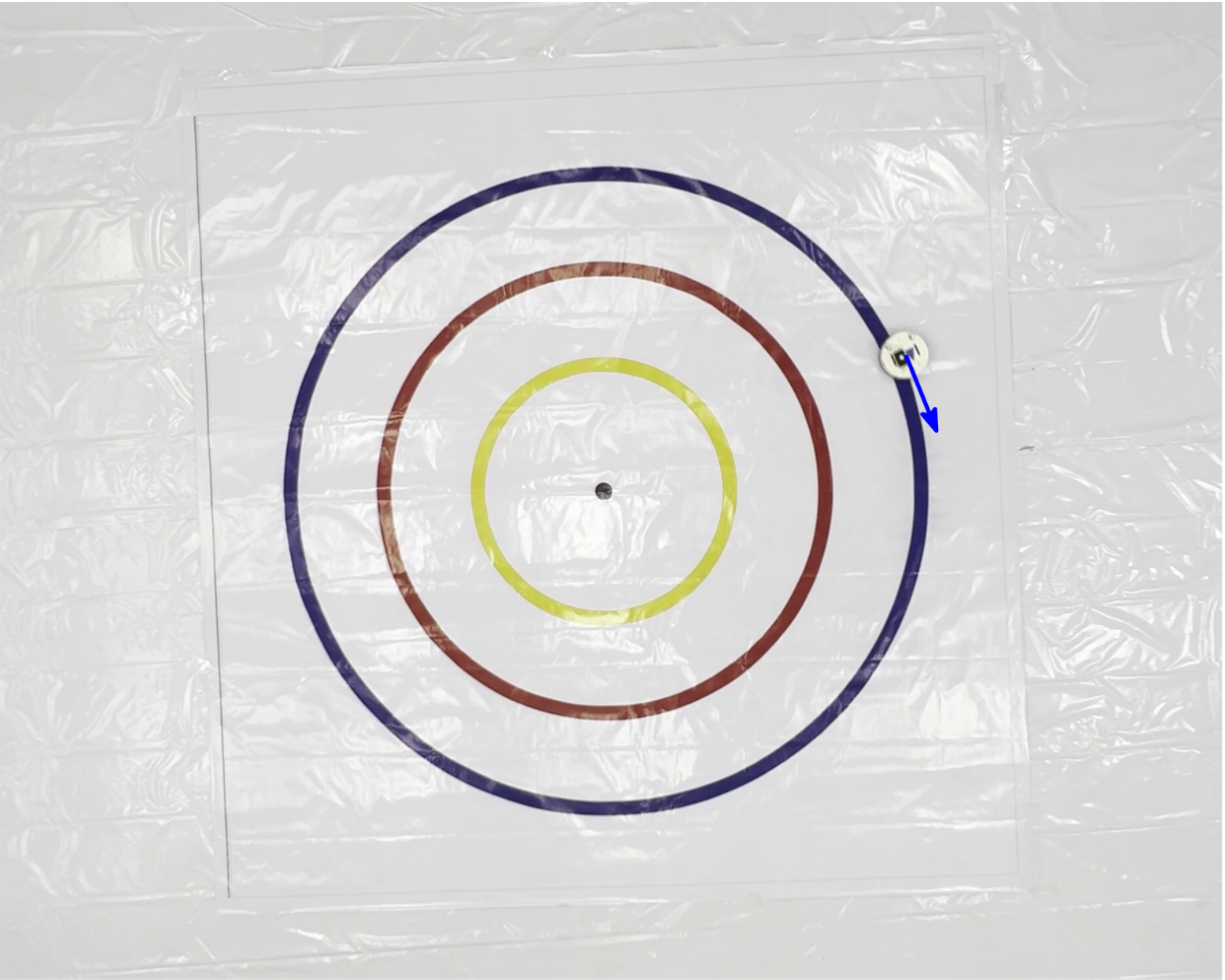}\label{multi_final}}
		\caption{Trajectory of Khepera IV robot in multiple entry case.}
		\label{fig_multi_entry_path}}
	\vspace*{-12pt}
\end{figure}

\begin{figure}[t!] 	
	\centering{
		\hspace*{-0.3cm}
		\subfigure[Range]{\includegraphics[width=3.4cm]{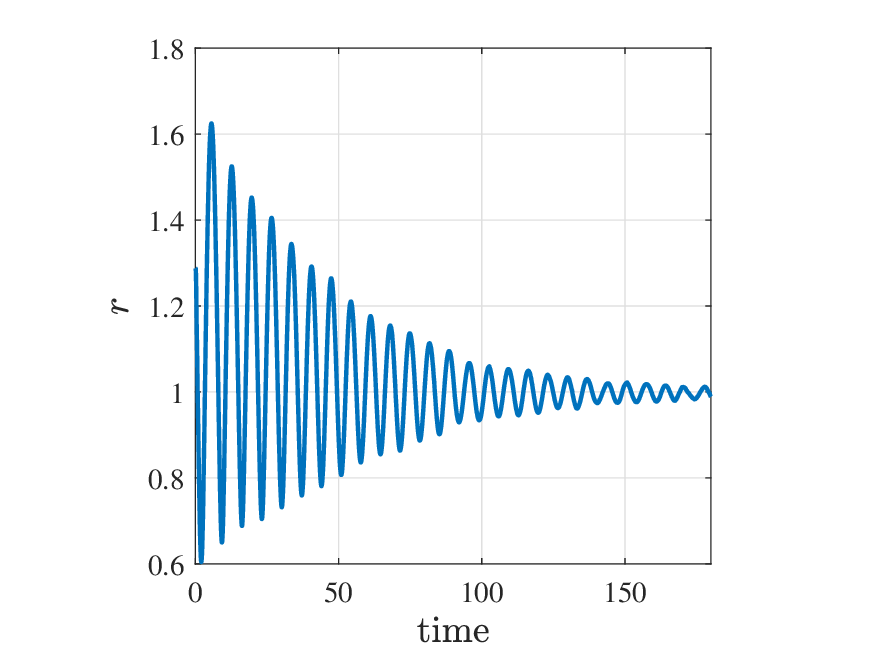}\label{multi_u}}\hspace*{-0.55cm}
		\subfigure[Bearing]{\includegraphics[width=3.4cm]{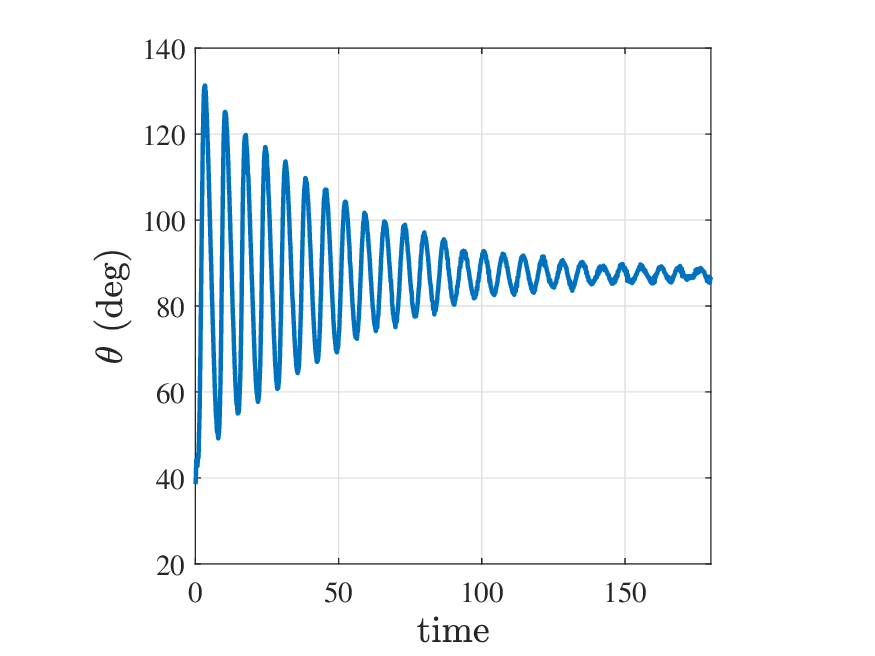}\label{multi_r}}\hspace*{-0.55cm}
		\subfigure[Control]{\includegraphics[width=3.4cm]{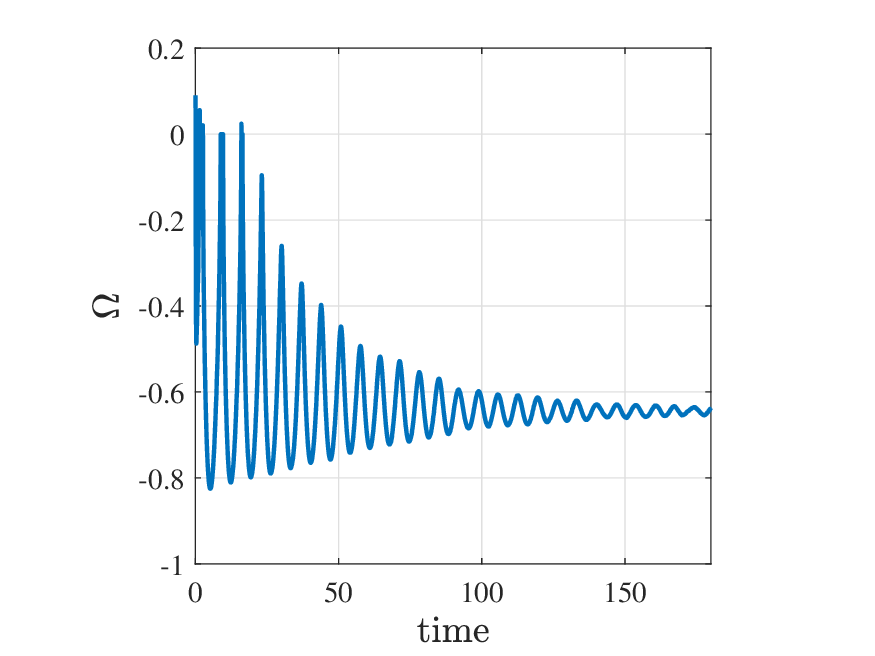}\label{multi_theta}}  
		\caption{Khepera IV robot with multiple entry case: Evolution of range, bearing angle and control.}
		\label{fig_multi_entry_data}}
	\vspace*{-15pt}
\end{figure}

\subsection{Experimental Results}
For experiments, we use the Khepera IV ground robot in the MoCap environment. The MoCap system comprises a track manager that can process the position and orientation data collected from the overhead cameras and provide feedback for implementing the proposed control algorithm. In our case, the Khepera IV robot acts as a rigid body in the track manager that streams the pose data to a Python script where we run our control algorithm. Further, the range-rate $\dot{r}$ is computed using Euler's approximation, which is run in the same loop where the controller $\Omega$ is evaluated. We randomly initialized $\dot{r}(0)$ within allowed limits (here, we considered $\dot{r}(0) = 0$) and calculated its subsequent values using the measurement of $r$, obtained from the feedback system, and the execution time of one control loop. Suppose that the value of $r$ measured in current iteration is $r(K)$ and the value stored from the previous iteration was $r(K-1)$, with the loop execution time measured being $T$, then we calculate $\dot{r}(K)$ as $\dot{r}(K) \approx [r(K) - r(K-1)]/{T}$. Since Khepera IV is (inherently) a differential drive robot, we resolve the linear and angular velocities of the unicycle model to individual wheel velocities of the Khepera IV as $V_r =  V + 0.5\Omega d_w, V_{\ell} = V - 0.5{\Omega}d_w$, where $V_r$ and $V_{\ell}$ are the speeds of the right and left wheels, respectively, $d_w = 10.54$ cm is the distance between the wheels, and $V$ and $\Omega$ are defined in \eqref{robot_model_cartesion}. The Khepera IV has a hardware limit for each wheel speed of $0.814$ m/s. Following this, we operated our robot at a linear velocity of $V = 0.6$ m/s and obtained $V_r$ and $V_{\ell}$, using \eqref{control_new}. We performed experiments for the same setting of the initial conditions and gain $\kappa$, as used for the simulation study. The robot trajectories for the case of multiple entry inside $\mathcal{C}_a$ are shown in Fig.~\ref{fig_multi_entry_path} with a complete trajectory profile captured from track manager as in Fig.~\ref{qtm_multi_final}. Further, the time evolution of different parameters is plotted in Fig.~\ref{fig_multi_entry_data}, indicating the desired behavior. The video of conducted experiments can be found on \url{https://youtu.be/mLBQcKGAsGQ}.   

\section{Conclusion and Further Remarks}\label{section_7}
We investigated the problem of safe circumnavigation around a hostile target using range and range-rate measurements. Benefiting from the straight line motion of the robot inside the auxiliary circle, we posed an equivalent problem of regulating the bearing angle at the entry point to this circle for achieving safe circumnavigation. The key technical developments in this direction were based on a crucial result (Theorem~\ref{thm_main_results}) where we leveraged the properties of BLF to design a parameter $\delta$. We determined an appropriate range for $\delta$ by applying Shafer's and logarithmic inequalities, showcasing its existence for the given radii of the three circles. It was demonstrated that the robot might enter the auxiliary circle multiple times (albeit finitely) before stabilizing to the desired circle, depending upon a gain term in the proposed controller. 

Our approach has a few limitations. First, the selection of the parameter $\delta$ within the range $\delta \in (0, 1)$ limits the bearing to the interval $(0, \pi)$. Secondly, our method depends on both range and range-rate information, but obtaining accurate range-rate information is not always straightforward or feasible \cite{cao2015uav,wang2024target}. These limitations present interesting avenues for future research where we aim to make contributions. Further, it will be interesting to investigate the problem using contemporary control barrier function (CBF)-based approaches, while relying solely on range-based measurements.


\bibliographystyle{IEEEtran}
\bibliography{References_Updated}


\end{document}